\documentclass{article}
\usepackage{kr}

\usepackage{times}
\usepackage{soul}
\usepackage{url}
\usepackage[hidelinks]{hyperref}
\usepackage[utf8]{inputenc}
\usepackage[small]{caption}
\usepackage{graphicx}
\usepackage{amsmath}
\usepackage{amsthm}
\usepackage{mathtools}
\usepackage{booktabs}
\usepackage{algorithm}
\usepackage{algorithmicx}
\usepackage{algpseudocode}
\usepackage{booktabs}
\usepackage[all]{foreign}
\newtheorem{example}{Example}
\newtheorem{proposition}{Proposition}

\newtheorem{theorem}{Theorem}

\title{Counterfactual Scenarios for Automated Planning}

\author{%
Nicola Gigante$^1$\and
Francesco Leofante$^2$\and
Andrea Micheli$^3$\\
\affiliations
$^1$Free University of Bozen-Bolzano, Bolzano, Italy\\
$^2$Department of Computing, Imperial College London, UK\\
$^3$Fondazione Bruno Kessler, Trento, Italy
\emails
nicola.gigante@unibz.it,
f.leofante@imperial.ac.uk,
amicheli@fbk.eu
}

\usepackage{amsmath,amsthm,amssymb}
\usepackage{mathtools}
\usepackage{float}
\usepackage[capitalise,noabbrev]{cleveref}
\usepackage{booktabs}
\usepackage{todonotes}
\usepackage{resizegather}
\usepackage{xspace}
\usepackage{blindtext}

\renewcommand{\implies}{\rightarrow}
\newcommand{\tuple}[1]{\ensuremath{\langle #1 \rangle}}
\newtheorem{definition}{Definition}

\newcommand{\pre}[1]{\ensuremath{\mbox{pre}_{#1}}}
\newcommand{\eff}[1]{\ensuremath{\mbox{eff}_{#1}}}

\newcommand{\allprobs}{\ensuremath{\mathbb{P}}\xspace}
\newcommand{\allpaths}[2][C]{\ensuremath{\Gamma_{#1}(#2)}\xspace}
\newcommand{\pathcost}[2][C]{\ensuremath{\Delta^{P}_{#1}(#2)}}
\newcommand{\minpath}[1]{\pathcost{#1}\xspace}

\newcommand{\plans}[1]{\ensuremath{\Pi_{#1}}}

\DeclareMathOperator*{\argmin}{arg\,min}

\newcommand{\explanation}[2]{\ensuremath{\Xi_{#1}(#2)}}

\newcommand{\ltlf}{LTL\ensuremath{_{f}}\xspace}

\newcommand\abs[1]{\left|#1\right|}

\newcommand\citet[1]{\citeauthor{#1}~(\citeyear{#1})}

\DeclareMathOperator{\csep}{CSEP}

\DeclareMathOperator{\enc}{enc}

\newcommand\tomorrow{\mathop{\bigcirc}}
\newcommand\wtomorrow{\mathrlap{\hspace{1pt}\sim}{\tomorrow}}
\newcommand\until{\mathrel{\mathcal{U}}}

\newcommand{\change}[1]{\textcolor{black}{#1}}

\begin{document}

\maketitle

\begin{abstract}
Counterfactual Explanations (CEs) are a powerful technique used to explain Machine Learning models by showing how the input to a model should be minimally changed for the model to produce a different output. Similar proposals have been made in the context of Automated Planning, where CEs have been characterised in terms of minimal modifications to an existing plan that would result in the satisfaction of a different goal. While such explanations may help diagnose faults and reason about the characteristics of a plan, they fail to capture higher-level properties of the problem being solved. 
To address this limitation, we propose a novel explanation paradigm that is based on counterfactual scenarios. In particular, given a planning problem $P$ and an \ltlf formula $\psi$ defining desired properties of a plan, counterfactual scenarios identify minimal modifications to $P$ such that it admits plans that comply with $\psi$. In this paper, we present two qualitative instantiations of counterfactual scenarios based on an explicit quantification over plans that must satisfy $\psi$. We then characterise the computational complexity of generating such counterfactual scenarios when different types of changes are allowed on $P$. We show that producing counterfactual scenarios is often only as expensive as computing a plan for $P$, thus demonstrating the practical viability of our proposal and ultimately providing a framework to construct practical algorithms in this area. 
\end{abstract}


\section{Introduction}

The widespread adoption of AI solutions for consequential decision-making tasks has fuelled considerable interest in Explainable AI (XAI). This is particularly true in the area of Machine Learning (ML), where black-box models are often deployed in applications such as credit risk analysis~\cite{HELOC} or bail approval~\cite{COMPASS}. In these settings, most approaches focus on explaining single-shot decisions (\ie predictions) produced by ML models but are often unable to explain more sophisticated decision-making tasks involving multiple reasoning steps.

The planning community has long recognised the importance of providing explanations for sequential decision-making tasks~\cite{ChakrabortiSK20}. As a result, a host of explainability approaches have been proposed. \change{Examples include model reconciliation~\cite{ChakrabortiSZK17,SreedharanCK18}, which focuses on resolving potential discrepancies between an AI agent's internal model and a human's mental model, and contrastive explanations~\cite{KrarupKMLC021,KrarupCL024,HoffmannM19}, which highlight differences between a plan generated by an AI and a user-suggested alternative, showing why the former was preferred.}

Much less attention has been devoted to counterfactual explanations, a popular explanation framework that is favoured in XAI due to its intelligibility and alignment with human reasoning~\cite{Byrne19}. Counterfactuals are actively studied in ML, where they are typically defined in terms of minimally altered inputs for which the ML model gives a different, more desirable output from that of the original input~\cite{KarimiBSV23}. Echoing this idea, recent work in planning proposed to characterise counterfactuals in terms of minimal modifications to an existing plan that would result in the satisfaction of a different goal~\cite{belle2023counterfactual}. 

Counterfactuals defined in this way are well suited to reason about \emph{local} properties of a plan, \ie diagnose faults in it and potentially identify avenues for repair showing how a plan would need to change for a desired outcome to be attained. However, in many practical applications, users are often interested in understanding how characteristics of a planning problem may affect the quality of plans that can be produced. In such cases, local counterfactual explanations are of little use as they fail to capture properties of the planning problem being solved, leaving many unanswered questions about the fundamental relationships between problem structure and plan properties.


\smallskip\textbf{Contributions.} In this paper, we fill this gap and propose a novel
explanation paradigm based on the concept of \emph{counterfactual scenarios} for
automated planning problems. Differently from existing proposals, our
explanations rely on counterfactual modifications of the planning problem
itself, thus pointing users of planning software to what would need to be
changed in the original problem formulation for a given course of action to be
observed. More formally, given a planning problem $P$ and an
\ltlf~\cite{DeGiacomoV13} formula $\psi$ defining desired properties of a
plan, counterfactual scenarios identify minimal modifications to $P$ such that
it admits plans that comply with $\psi$. We investigate the computational
complexity of generating such explanations under counterfactual modifications
that can alter the initial state of $P$, the structure of its actions or its
goals. We focus our analysis on two qualitative variants of counterfactuals
scenarios: \emph{existential counterfactual scenarios} ($\exists$), which
identify modifications to $P$ such that it admits \emph{at least one plan}
satisfying $\psi$, and \emph{universal counterfactual scenarios} ($\forall$),
which modify $P$ in such a way that \emph{all its valid plans} satisfy
$\psi$. We demonstrate the feasibility of our proposal by showing
that generating $\exists$ and $\forall$ counterfactual scenarios often has a
computational complexity comparable to finding a plan for the original problem,
thus providing a foundation for developing practical algorithms in this domain.


% ...

\smallskip\textbf{Related work.} The challenge of explaining AI behaviours in sequential decision-making settings has gained considerable attention (see, \eg~\cite{baier_et_al:DagRep.14.9.67} for a recent account). As a result, a host of approaches have been proposed both in model-based and model-free settings (see, \eg~\citet{ChakrabortiSK20} and \citet{MilaniTVF24} for a general overview). \change{In the following, we focus on counterfactual explanations for sequential decision-making tasks, but refer the reader to \cref{sec:discussion} for a broader discussion of related work in the planning literature.}

While research on counterfactuals for sequential decision-making is still in its
infancy, several proposals have already been put forward. For instance,
counterfactuals for automated planning problems have been defined
by~\citet{belle2023counterfactual} as minimal modifications to an existing plan
that would result in the satisfaction of a different goal. This definition
echoes existing characterisations of counterfactuals in
ML~\cite{wachter2017counterfactual} and is closely related to the literature on
plan repair~\cite{FoxGLS06}. However, differently from the aforementioned
definition, our counterfactuals identify changes
\emph{in the planning problem itself} and can thus be used to point to how
changes in the structure of a problem may affect plan properties. 

Counterfactuals have also been investigated in the context of sequential decision-making problems modelled by Markov Decision Processes (MDPs). For instance,~\citet{Kobialkaetal} studied the problem of generating counterfactual strategies for MDPs by computing minimal modifications to an original strategy that would lead to reaching desirable states high probability. Similar efforts have been made in the Reinforcement
Learning arena, as summarised by~\citet{GajcinD24}. While the works above propose notions that vary
based on the specific changes that can be effected on a policy, they still frame
counterfactuals as minimally altered policies, differently from our
counterfactual scenarios. Finally, causality-based counterfactuals for MDPs have been investigated~\cite{TsirtsisDR21,kazemi2024counterfactual}. In
this line of work, counterfactual explanations are defined in terms of paths
that diverge by at most $k$ actions from a given initial path in the MDP, again
marking a key difference from our proposal.

\smallskip\textbf{Structure of the paper.} In the next section, we motivate and explain
the need for counterfactual explanations whose scope extends beyond existing
proposals, while \cref{sec:preliminaries} provides some background concepts on planning and \ltlf.
Then, \cref{sec:formalization} presents our novel notion of counterfactual
scenarios, and in \cref{sec:classes} we present three concrete classes of
counterfactuals that we are interested in analysing. \Cref{sec:results} proves
computational complexity results for such classes. Finally, we discuss the implications of our results and draw our
conclusions in \cref{sec:discussion} \change{and \cref{sec:future}}.









\section{A motivating example}
\label{sec:message}

To see what makes counterfactual scenarios interesting and useful, let us
consider an example based on the food delivery domain originally formulated
by~\citet{KrarupCL024}.

\begin{figure}
    \centering
    \includegraphics[width=\linewidth,trim={12cm 7cm 15cm 6cm},clip]{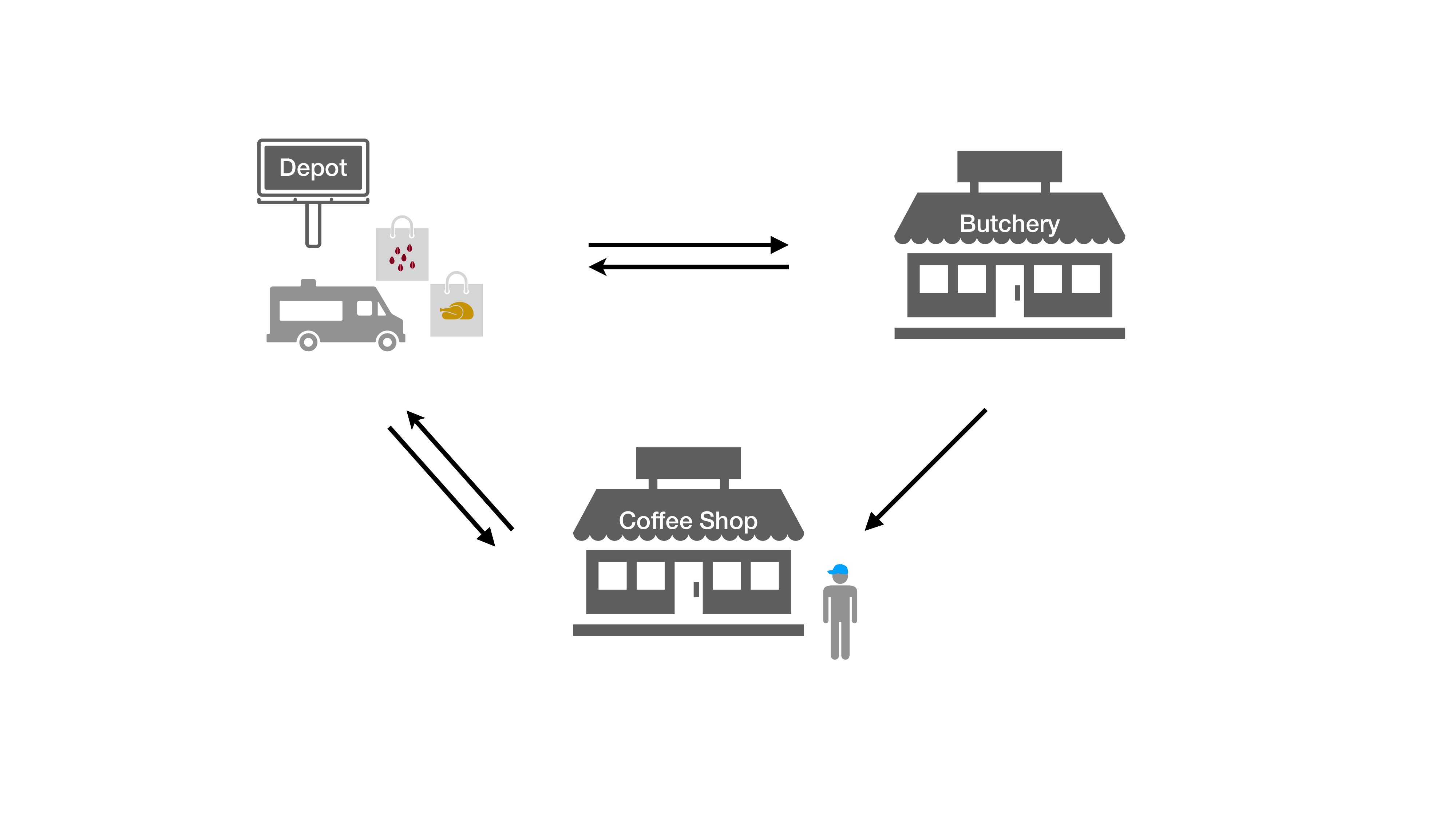}
    \caption{The food delivery domain.}
    \label{fig:overview}
\end{figure}

\begin{example}
\label{ex:running}
    The domain features one driver, one truck and three locations: a depot, a butchery, and a coffee shop. All locations are connected by roads with no pedestrian access, modelled by a predicate \texttt{link(?from,?to)} (links are directed). Both the goods and the truck are initially located at the depot (\texttt{at(coffee,depot)}, \texttt{at(meat,depot)}, \texttt{at(truck,depot)}), while the 
driver is located at the coffee shop (\texttt{at(driver,coffee shop)}). The goal is for the driver to deliver meat to the butchery and coffee beans to the coffee shop ($\texttt{at(meat,butchery) } \wedge \texttt{ at(coffee,coffee shop)}$). To achieve this goal, the following actions can be performed:
\begin{itemize}
    \item \texttt{load(?item,?truck,?loc)}: loads an item onto the truck, provided both are at the same location. When an item is loaded, it is removed from its location and is marked as stored inside the truck ($\neg\texttt{at(?item,?truck)} \wedge \texttt{in(?item,?truck)}$);
    \item \texttt{unload(?item,?truck,?loc)}: removes an item from the truck, making it available at the location where it was unloaded;
    \item \texttt{drive(?truck,?from,?to,?driver)}: allows to drive the truck between two specified locations, provided that they are connected by a road, \ie \texttt{link(?from, ?to}) is true. The driver and the truck must start from the same location for this action to be executable.
\end{itemize}
A visual representation of this problem is shown in Figure~\ref{fig:overview}.
\end{example}

Existing approaches to define counterfactual explanations for sequential
problems assume the availability of an initial action sequence from which the
counterfactual explanation can be obtained. However, there may be cases where
the planning problem does not admit a solution and the problem of generating
counterfactuals for such instances becomes ill-defined. For instance, the
planning problem described in \cref{ex:running} does not admit a
solution, as the driver has no way to reach the truck and carry out deliveries.
As a consequence, no initial action sequence can be determined to seed the
search for a counterfactual explanation. In situations where a solution is
unattainable, our \emph{existential counterfactual scenarios} can help identify
modifications to the planning problem that would make the problem solvable, as
demonstrated in \cref{ex:existential}

\begin{example}
\label{ex:existential}
    The planning problem described in \cref{ex:running} does not admit a solution, as the driver is initially located away from the truck with no way to reach it. In such cases, an existential counterfactual scenario highlights how the initial problem should be changed for it to admit a solution. A possible existential counterfactual scenario can therefore be obtained by changing the initial location of the driver from \texttt{at(driver, coffee-shop)} to \texttt{at(driver, depot)}. As a result, the new planning problem obtained by implementing this change admits (at least) one plan: 
\begin{itemize}
\small
    \item \texttt{load(meat,truck,depot)}
    \item \texttt{load(coffee,truck,depot)}
    \item \texttt{drive(truck,depot,butchery,driver)}
    \item \texttt{unload(meat,truck,butchery)}
    \item \texttt{drive(truck,butchery,coffee-shop,driver)}
    \item \texttt{unload(coffee,truck,coffee shop)}
\end{itemize}
\end{example}

The counterfactual scenario in \cref{ex:existential} helps users of planning software deal with unsolvable planning problems, potentially identifying changes in the problem formulation that determined unsolvability. These kind of explanations can be used to complement existing efforts, \eg~\cite{ErikssonRH18,ErikssonH20} and improve the user experience, debugging process, or overall effectiveness of planning tools. However, existential counterfactual scenarios also prove useful when planning problems are solvable to begin with, as discussed in \cref{ex:existential_pre}.

\begin{example}
    \label{ex:existential_pre}
    Consider the planning problem given by the existential counterfactual scenario in \cref{ex:existential}. Suppose now that the user requires that in some plans, the truck visits the coffee shop before making a delivery to the butchery. This can be captured by the following logical specification:
\begin{gather*}
    \begin{aligned}
        \psi \coloneqq \mathop{\Box} & \bigl(
        \mathtt{at(truck, coffee{-}shop)} \implies \tomorrow \mathtt{at(truck,butchery)}\bigr)\\ {}\land{}
                & \bigl(\neg \mathtt{at(truck,butchery)} \until \mathtt{at(truck,coffee{-}shop)}\bigr)
    \end{aligned}
\end{gather*}

    Satisfying this requirement is impossible in the original domain, as driving from the coffee shop to the butchery would require a (directed) connection between the two locations. This fact can be highlighted by means of an existential counterfactual scenario, showing how the preconditions of the action \texttt{drive} should be minimally changed for some valid plans to satisfy $\psi$. A possible solution to the problem corresponds to a modified planning problem where the precondition of the action is weakened to:
    \begin{equation*}
        \left\{
            \begin{aligned}
                & \mathtt{at(driver,coffee{-}shop)} \land {}\\
                & \mathtt{at(truck,coffee{-}shop) } \land {}\\
                & \mathtt{link(coffee{-}shop,butchery)}
            \end{aligned}
        \right\}
    \end{equation*}

    With this change, the truck is allowed to visit the butchery after the coffee shop in some of the plans admitted by the existential counterfactual scenario.
   
\end{example}

\cref{ex:existential_pre} showed how existential counterfactual scenarios can result in changes that bring about \emph{at least one} plan satisfying user requirements. However, depending on the application, users may be interested in changes that constrain the planning problem to have \emph{all} its valid solutions comply with a specification. Our \emph{universal counterfactual scenarios} can achieve this, as demonstrated in \cref{ex:universal}.

\begin{example}
    \label{ex:universal}
    Consider the existential counterfactual scenario from \cref{ex:existential}. Suppose now that the user requires that in all plans, the truck should always go back to the depot (eventually) after completing all deliveries. This can be expressed by the following logical specification:
    \begin{equation*}
        \psi := \Box \left[
            \begin{aligned}
                & \left(\begin{aligned}
                    & \mathtt{at(meat,butchery)} \land {} \\
                    & \mathtt{at(coffee,coffee{-}shop)} 
                \end{aligned}\right)\implies\\
                \implies & \,\,\Diamond \mathtt{at(truck,depot)}    
            \end{aligned}
        \right]
    \end{equation*}
    
    Universal counterfactual scenarios can help identifying changes in the planning problem that would satisfy the requirement above. For instance, strengthening the goal to:
    \begin{equation*}
        \left\{
            \begin{aligned}
                & \mathtt{at(meat,butchery)} \land {}\\
                & \mathtt{at(coffee,coffee{-}shop) } \land {} \\
                & \mathtt{at(truck,depot)}
            \end{aligned}
        \right\}
    \end{equation*}
    would result in the the truck eventually visiting the depot in all plans
    admitted by the modified planning problem.
\end{example}

This section introduced the intuition behind counterfactual scenarios,
highlighting some use cases. What follows will formally define our explanation
framework and present a detailed analysis of the computational complexity of
deciding the existence of such explanations.


\section{Preliminaries}
\label{sec:preliminaries}

For the sake of this paper, we focus on classical planning, adopting a general
formulation with arbitrary formulae as preconditions. While the definitions below easily generalize to numeric planning, we restrict ourselves to
the classical, finite-state case because numeric planning is in general
undecidable thus making all the explainability problems addressed in this paper
trivially undecidable as well.

\begin{definition}[Planning problem]
    \label{def:planning:problem}
A (ground) \textbf{planning problem} is a tuple $P=\tuple{F, A, I, G}$ where:
\begin{enumerate}
\item $F$ is a set of Boolean fluents;
\item $A$ is a set of actions, where each $a \in A$ comes with:
\begin{enumerate}
    \item a precondition $\pre{a}$ expressed as a formula over $F$; and
    \item a set of effects $\eff{a}$ of the form $f := v$ where $f \in F$ and $v \in \{\top, \bot\}$;\footnotemark
\end{enumerate} 
\item $I: F \rightarrow \{\top, \bot\}$ is the initial state encoded as a total assignment of truth values to fluents in $F$;
\item $G$ is the goal condition expressed as a formula over $F$.
\end{enumerate}
\end{definition}

\footnotetext{For the sake of simplicity, we assume that in every action there is at most one effect for a fluent $f$.}


\noindent
We indicate with $\mathcal{P}$ the set of all possible planning problems.

\begin{definition}[Plan state]
A \textbf{state} for a planning problem $P=\tuple{F, A, I, G}$ is a total assignment $s: F \rightarrow \{\top, \bot\}$ of values to fluents.
\end{definition}

\begin{definition}[Applicable action]
Given a planning problem $P=\tuple{F, A, I, G}$, an action $a \in A$ is \textbf{applicable} in state $s$ if the formula \pre{a} is satisfied by the assignment of $s$.
\end{definition}

\begin{definition}[Successor state]
Given a planning problem $P=\tuple{F, A, I, G}$ and a state $s$, the \textbf{successor $a(s)$ of an applicable action} $a \in A$ is a state such that:

\begin{equation*}
    a(s)(f) \coloneqq 
    \begin{cases}
        v & \text{if $f := v \in \eff{a}$} \\
        s(f) & \text{otherwise}
    \end{cases}
\end{equation*}
\end{definition}

\begin{definition}[Sequential plan]
\label{def:plan}
A \textbf{sequential plan} for a planning problem $P=\tuple{F, A, I, G}$ is a
sequence of actions $\pi=\tuple{a_0, \ldots, a_{n-1}}$ where $a_i \in A$. When
denoting $s_0 = I$ and $s_{i+1} = a_{i}(s_{i})$, we say $\pi$ is \textbf{valid}
if every $a_i$ is applicable in $s_{i-1}$ and $s_{n} \models G$.
\end{definition}

Given a planning problem $P = \tuple{F, A, I, G}$, we write $\plans{P}$ for the
set of all valid plans. Note that there is no reason for a plan in
\emph{classical} planning for visiting twice the same state, as any such plan
can be cut into a shorter one that does not. For this reason, we assume
\emph{w.l.o.g.}\ all plans in $\plans{P}$ to be devoid of state loops.

Our counterfactual scenarios are based on the notion of satisfaction of an \ltlf
formula~\cite{DeGiacomoV13}. \ltlf is a propositional modal logic interpreted
over \emph{finite words}. We consider here the standard syntax of \ltlf where a
formula $\psi$ is defined by the following grammar:
\begin{align*}
    \psi \coloneqq p & {} \mid \neg\psi \mid \psi\lor\psi \mid\psi\land\psi \\
                     & {} \mid \tomorrow \psi \mid \wtomorrow \psi \mid
                     \psi \mathrel{\mathcal{U}} \psi \mid 
                     \mathop{\Diamond} \psi \mid \mathop{\Box}\psi
\end{align*}
where $p\in F$ with $F$ being some finite set of \emph{propositions}. 

The semantics is defined as usual in the literature (we omit the formal
definition because of space concerns), where: $\tomorrow \psi$ and
$\wtomorrow\psi$ are the \emph{strong} and \emph{weak tomorrow} operators which
state the truth of $\psi$ at the next state (only if it exists, in the case of
the weak operator); $\psi_1 \mathrel{\mathcal{U}} \psi_2$ is the \emph{until}
operator stating that $\psi_2$ will happen in the future and $\psi_1$ holds
everywhere until then; $\mathop{\Diamond} \psi$ and $\mathop{\Box}\psi$ are the
\emph{eventually} and \emph{globally} operators, definable respectively as
$\mathop{\Diamond} \psi \equiv \top \mathrel{\mathcal{U}} \psi$ and
$\mathop{\Box}\psi\equiv\neg\Diamond\neg\psi$, which mandates $\psi$ to hold
somewhere in the future or always in the future.

We recall that the \emph{satisfiability} and \emph{validity} problems for \ltlf,
\ie deciding respectively whether there exists a word satisfying a given formula
or whether all words satisfy a given formula, are both
PSPACE-complete~\cite{DeGiacomoV13}, which also happens to be the same
computational complexity of the plan existence problem in classical
planning~\cite{Bylander94}. 

With a slight abuse of notation, we say that a plan $\pi$ satisfies an \ltlf
formula $\psi$ if $\tuple{s_0, s_1, \ldots, s_n} \models \psi$, where the $s_i$
are defined as in \cref{def:plan}. 

\section{Abstract Counterfactual Scenarios}
\label{sec:formalization}

In this section, we introduce a general and abstract formulation of our
counterfactual scenarios which can be instantiated into several different kind
of explanations depending on the case at hand. In \cref{sec:classes} we
instantiate the framework considering some interesting concrete classes of
counterfactuals which then will be studied in \cref{sec:results} from a
computational complexity perspective.

In the following, let $\allprobs$ denote the class of all the possible
\emph{planning problems} as in \cref{def:planning:problem}.


\begin{definition}[Possible changes]
    A \textbf{possible-change relation} is a relation $C \subseteq \allprobs
    \times \allprobs$.
\end{definition}

Intuitively, a possible-change relation describes abstractly which classes of
\emph{changes} are admitted. Then, a \emph{counterfactual scenario} can be
defined as a planning problem that can be obtained by a minimal sequence of
admitted changes. This notion is formally defined by introducing an abstract
\emph{transition system} whose worlds are planning problems.

\begin{definition}[Counterfactual transition system]
    Let $P$ be a classical planning problem, and let $C$ be a possible-change
    relation. A \textbf{counterfactual transition system} for $P$ and $C$ is the
    transition system $\mathcal{C}(P,C) \coloneqq \tuple{\allprobs, P, C}$
    where:
    \begin{enumerate}
        \item $\allprobs$ is the set of worlds of the transition system;
        \item $P$ is the initial world of the system.
        \item $C$ is the accessibility relation of the system;
    \end{enumerate}
\end{definition}

A path $\gamma$ in $\mathcal{C}(P,C)$ is a sequence of problems $\tuple{P_0,
P_1, \ldots, P_n}$ where $P_0 = P$ and $(P_i, P_{i+1})\in C$ for each $i \in [0,
n-1]$. The \emph{cost} of a path is its length $\abs{\gamma}$.
    
Given a planning problem $P'$ and a possible-change relation $C$, we denote as
$\allpaths{P'}$ all the paths ending in $P'$ in $\mathcal{C}(P,C)$. We denote as
$\minpath{P'}=\min_{\gamma \in \allpaths{P'}} \abs{\gamma}$ the minimum cost of
reaching $P'$.

\begin{definition}
\label{def:problem}
    Given a planning problem $P = \tuple{F, A, I, G}$, an \ltlf formula $\psi$ defined over the alphabet $F$ and a possible-change relation $C$, an \textbf{existential counterfactual scenario} is a planning problem $\explanation{\exists}{P, \psi, C}$ such that:
    \begin{equation*}
        \explanation{\exists}{P, \psi, C} \coloneqq 
        \argmin_{\substack{
            \text{$P' \in \allprobs$ s.t.}\\
            \text{$\exists \pi' \in \plans{P'}$ s.t.}\\
            \pi' \models \psi
        }} \minpath{P'}
    \end{equation*}

    \noindent
    Similarly, a \textbf{universal counterfactual scenario} is a planning problem $\explanation{\forall}{P, \psi, C}$ such that:
    \begin{equation*}
    \explanation{\forall}{P, \psi, C} \coloneqq 
    \argmin_{\substack{
        \text{$P' \in \allprobs$ s.t.}\\
        \text{$\plans{P'} \not= \emptyset$ and}\\
        \forall \pi' \in \plans{P'} \pi' \models \psi
    }} \minpath{P'}
    \end{equation*}
\end{definition}

\noindent
Note that existential and universal counterfactuals might not exist if there is
no problem $P'$ reachable in $\mathcal{C}(P,C)$ admitting a valid plan that
satisfies $\psi$.

Moreover, our definitions do not rule out the possibility that a counterfactual scenario for a problem $P$ is $P$ itself. However, this behaviour is unlikely to happen in practical settings and depends on the type of counterfactual scenario being sought. Typically, users of planning tools are interested in explaining a plan $\pi$ \emph{after} it has been generated by a planner. When searching for an existential counterfactual scenario, it is possible for $\explanation{\exists}{P, \psi, C}$ to be equal to $P$, because the planner might have picked a plan not satisfying $\psi$ non-deterministically from the pool of existing plans. However, if the user is interested in universal counterfactual scenarios, then $\explanation{\forall}{P, \psi, C}$ for $P$ cannot be $P$ itself since the existence of $\pi$ would violate Definition~\ref{def:problem}.

\section{Concrete Classes of Counterfactuals}
\label{sec:classes}

The definitions of existential and universal counterfactual scenarios given in
the previous section are intentionally quite abstract and general, and in
applications they need to be instantiated over concrete classes of
possible-change relations. In this section, we
describe three possible such instantiations, which define as many different
classes of counterfactual scenarios that are interesting in practice. Then,
\cref{sec:results} will study the associated decision problems from a
complexity-theoretic standpoint.
In particular, we focus on counterfactuals where we allow changes in the
\emph{initial state}, the \emph{goal conditions} or the \emph{action
preconditions}. For the sake of this paper, we exclude the possibility of
modifying the effects of actions for two reasons. First, for counterfactuals to
be useful in practice, they need to be \emph{plausible}: they should represent
scenarios that could realistically occur or be implemented, rather than propose
unachievable changes to an action's effects (\eg, proposing the complete removal
of all resource consumption following the execution of an action, which might be
practically impossible).
Second, from a technical standpoint, one needs to be very careful when choosing
which effects are allowed to be changed to avoid constructing trivial
counterfactuals (\eg, achieving the goal from the first action by adding a goal
state itself goal as an effect). Nonetheless, the framework introduced in the
previous section can easily accommodate the change of effects and one can
describe the plausibility constraints in the possible-change relation.

We start from the simplest instance of our framework, where only changes to the
\emph{initial state} are allowed. To formalise this requirement, we define a
possible-change relation that admits problems that differ for the initial state
only. Let $F$ be a set of fluents, then:

\begin{equation*}
C_{\mathit{init}} \coloneqq \left\{
    (P,P') \middle| 
    \begin{array}{@{}c@{}}
        P=\tuple{F, A, I, G},\,P'=\tuple{F, A, I', G}\\[0.5em]
        \text{$I=I'$ except at most one fluent}
    \end{array} 
\right\}
\end{equation*}

Note that single edits can change the assignment of at most one single fluent in
the initial state, so the distance between two different initial states is the
minimum number of edits needed to change one into the other. 

A slightly more involved definition arises when one wants to admit changes to
the \emph{goal conditions}. In contrast to initial states, goals are expressed
as arbitrary Boolean formulas, therefore a slight change may be more semantically
relevant than it appears. To capture this nuance, when comparing two different
goal conditions we count the number of \emph{differing truth assignments} of the
conditions, and we only admit to add or remove one truth assignment at a time,
obtaining the following definition:

\begin{equation*}
C_{\mathit{goal}} \coloneqq \left\{
    (P, P') \middle| 
    \begin{array}{@{}c@{}}
        P=\tuple{F,A,I,G},\, P'=\tuple{F,A,I,G'}\\[0.5em]
        \#\bigl[(G \wedge \neg G') \vee (G' \wedge \neg G)\bigr] = 1
    \end{array}
    \right\}
\end{equation*}
where $\#[\phi]$ is the number of \emph{truth assignments} of $\phi$.

Note that the definition of $C_{\mathit{goal}}$ considers both changes that
\emph{strengthen} the goal condition, \ie that remove goal states, and those
that \emph{weaken} the condition, \ie that add goal states. However, note that,
for \emph{existential} counterfactuals, \emph{strengthening} the goal is never
useful, because if a plan satisfying $\psi$ does not exist, \emph{removing} goal states will not help. Conversely, \emph{weakening} the goal is always sufficient, because adding a state reachable by a valid plan
always costs less than adding it and removing something else.

For \emph{universal} counterfactuals, instead, \emph{weakening} is never useful,
because if a plan \emph{not} satisfying $\psi$ exists, it cannot disappear by
\emph{adding} goal states. However, note that strengthening the goal is not the
only option in the universal case, because a counterfactual may be found by
considering an entirely disjoint set of goal states.

A class of counterfactual scenarios that is apparently similar to the latter is
that obtained by admitting changes to \emph{actions' preconditions}. In this
case we define a possible-change relation that admits only problems where the
only change is the precondition of a single action where a single truth
assignment has been added or removed.

\begin{equation*}
C_{\mathit{act}} \coloneqq \left\{
    \!(P, P') \middle| 
    \begin{array}{@{}c@{}}
        P=\tuple{F, A, I, G},\,P'=\tuple{F, A', I, G}\\[0.5em]
        A \setminus A' = A' \setminus A = \{a\}\\[0.5em]
        \eff{a} = \eff{a'}\\[0.5em]
        \#\left[{} \lor {} 
            \begin{aligned}
                & (\pre{a} \land \neg \pre{a'})\\
                & (\pre{a'} \land \neg \pre{a})
            \end{aligned}
        \right] = 1
    \end{array}
\right\}    
\end{equation*}

Note that for this class of counterfactuals we can observe, similarly to the
previous case, that in the \emph{existential} case only \emph{weakening} the
preconditions make sense. Although the definitions of $C_{\mathit{goal}}$ and
$C_{\mathit{act}}$ present many similarities, shortly we will see in
\cref{sec:results} that they are computationally quite different.

\smallskip\textbf{Defining plausible scenarios.} In concrete applications, not
all possible changes in either the initial state, the goal or the action
preconditions might be acceptable. For instance, in Example~\ref{ex:running},
\texttt{at(truck, depot)} and \texttt{at(truck, butchery)} cannot be both
initially true, otherwise the model would allow the truck to be in two locations
at the same time. Hence, one needs to impose additional \emph{plausibility
constraints} to control which changes are allowed to the planning problem. This
can be easily achieved by leveraging the expressive power of \ltlf
. In the following, let $\phi_{init}, \phi_{act}, \phi_{goal}$ denote three propositional formulae defining plausibility constraints for the initial state, action preconditions and goal conditions respectively.
To limit the changes in $C_{\mathit{init}}$, we simply conjoin the plausibility constraints (without temporal operators) in $\phi_{init}$ with $\psi$; in this way, we obviously constrain the possible initial states, because of the semantics of \ltlf. 
For $C_{\mathit{goal}}$, we conjoin the formula $\Diamond(\wtomorrow \bot \land \phi_{goal})$ into $\psi$. Essentially, we are forcing $\phi_{goal}$ to hold on all the final states of a plan (where $\wtomorrow \bot$ holds).
Finally, if we want to impose a propositional plausibility constraint on the
precondition of an action $a$, we assume \emph{w.l.o.g.}\ that there exists a
fluent $f_a$ that is set to true by the effects of action $a$ and falsified by
the effects of all the other actions. Then, we add the formula $\Box(\tomorrow f_a
\rightarrow \phi_{act})$, forcing the formula $\phi_{act}$ to hold in the state
where action $a$ has been applied.
In all three cases above, imposing plausibility constraints can be done by simply extending the formula $\psi$ without complicating the problem further. As such, we do not need to explicitly consider plausibility constraints in the theoretical analysis presented in the next section.

\section{The Hardness of Finding Counterfactuals}
\label{sec:results}

In this section we study the \emph{computational complexity} of finding
counterfactuals of the kinds introduced in \cref{sec:classes}. We will see that
the decision problem for most of them surprisingly remains PSPACE-complete,
which is the same complexity as finding a plan in the first place.
A complete picture of the results is available in \cref{tbl:results-overview}.

\begin{table}[t]
    \centering
    \renewcommand{\arraystretch}{2}
    \begin{tabular}{@{} c c c c @{}}\toprule
    & \bfseries Initial State & \bfseries Goals & \bfseries Action Preconditions \\[-0.8em]
    & $C_{\mathit{init}}$ & $C_{\mathit{goal}}$ & $C_{\mathit{act}}$ \\
    \midrule
    \bfseries $\Xi_{\exists}$ & PSPACE-c. & PSPACE-c. & PSPACE-c. \\[-1em]
     & \small \cref{thm:initial} & \small \cref{thm:goals:existential} & \small \cref{thm:acts:existential}\\
    \midrule
    \bfseries $\Xi_{\forall}$ & PSPACE-c. & PSPACE-c. & $\in$ NEXP\textsuperscript{NP} \\[-1em]
     & \small \cref{thm:initial} & \small \cref{thm:goals:universal} & \small \cref{thm:acts:universal:membership}\\
    \bottomrule
    \end{tabular}
    \caption{Overview of the computational complexity results.}
    \label{tbl:results-overview}
\end{table}

In our framework, the quality of plans is evaluated in terms of the
satisfaction of an \ltlf formula $\psi$ over the fluents $F$. Therefore, we need
a way to link the two worlds of planning and \ltlf. To this end, we can leverage a well-known result stating that any classical planning problem can be translated into a suitable
\ltlf formula that represents the plans of the problem and the states visited by
those plans.
\begin{proposition}[\citet{CialdeaMayerLOP07}]
    \label{prop:plantoltl}
    For any classical planning problem $P=\tuple{F,A,I,G}$ there exist:
    \begin{enumerate}
        \item an \ltlf formula $[P]$ over the alphabet $A\cup F$ such that from
            any model of $[P]$ one can extract in polynomial time a plan for 
            $P$ (not necessarily reaching the goal); 
        \item an \ltlf formula $[P]_G$ over the alphabet $A\cup F$ such that
            from any model of $[P]_G$ one can extract in polynomial time a 
            solution plan for $P$ (\ie reaching the goal); 
    \end{enumerate}
\end{proposition}

The extraction in polynomial time is actually a simple projection of the
propositions from $A$ in the model of the \ltlf formula: if proposition $a$ is
true at a given time point it means action $a$ is executed in that step.
Similarly, the propositions from $F$ encode the states visited by the plan.

Since the counterfactual scenarios are \emph{minimal} objects, minimizing the
number of edits, finding one is an \emph{optimization} problem. Therefore, as is
customary in complexity theory, our analysis is concerned with the
\emph{decision problem} associated with such an optimization problem.
\begin{definition}[Counterfactual scenario existence problem]
    \label{def:decision:problem}
    Let $C\subseteq \allprobs \times \allprobs$ be a possible-change relation.
    The \textbf{existential} \textbf{counterfactual scenario existence problem}
    under $C$, denoted $\csep_\exists(C)$, is the problem of deciding, given a
    planning problem $P$, an \ltlf formula $\psi$, and a budget $K\ge0$, whether
    $\explanation{\exists}{P, \psi, C}$ exists and
    $\pathcost{\explanation{\exists}{P, \psi, C}}\le K$.
    
    The \textbf{universal} counterfactual scenario existence problem under $C$,
    denoted $\csep_\forall(C)$, is defined similarly over
    $\explanation{\forall}{P, \psi, C}$.
\end{definition}

Our first result is that counterfactual scenario existence problems are, in
general, PSPACE-hard. 

\begin{theorem}
    \label{thm:hardness}
    Let $F$ be a set of fluents and $C\subseteq \allprobs \times \allprobs$ be a
    possible-change relation over $F$. Then, both $\csep_\exists(C)$ and
    $\csep_\forall(C)$ are \textbf{PSPACE-hard}.
\end{theorem}
\begin{proof}
    We go by reduction from the plan existence problem of classical
    planning which is known to be PSPACE-complete~\cite{Bylander94}. Given
    $P=\tuple{F,A,I,G}$, it is sufficient to ask for an existential or universal
    counterfactual scenario for a trivial formula $\psi\coloneqq\top$ and a cost
    $K=0$. The formula is trivially always satisfied, and the zero cost ensure
    to be unable to make any change to $P$. Then, since both the existential and
    the universal definitions require the counterfactual to have at least a
    solution, the counterfactual exists if and only if $P$ has a solution in the
    first place.
\end{proof}

Following \cref{thm:hardness}, we can now focus specifically on showing that the
decision problems for most of our particular classes of counterfactual scenarios
\emph{belong} to PSPACE and are therefore PSPACE-complete.

As in \cref{sec:classes}, we start from the class of
counterfactuals where only changes to the \emph{initial state} are admitted.

\begin{theorem}
    \label{thm:initial}
    Both $\csep_\exists(C_{\mathit{init}})$ and
    $\csep_\forall(C_{\mathit{init}})$ are \textbf{PSPACE-complete}.
\end{theorem}
\begin{proof}
    Let $P=\tuple{F,A,I,G}$, $\psi$ an \ltlf formula, and let $K\ge0$ be the
    cost. We proceed as follows. If $K=0$, we define $I'=I$. Otherwise, we loop
    through all the possible alternative initial states $I'$ which differs from
    $I$ of at most $K$ fluents. In both cases, for all the considered $I'$, we
    test the satisfiability of $\psi\land[P]_G$, in the case of
    $\csep_\exists(C_{\mathit{init}})$, or the validity of $[P]_G\to\psi$ in
    the case of $\csep_\forall(C_{\mathit{init}})$ (see also
    \cref{prop:plantoltl}). If we find any $I'$ for which the test is
    successful, we found an existential (or universal) counterfactual, otherwise
    we reply negatively. Since satisfiability and validity for \ltlf is
    PSPACE-complete, and we only need an additional polynomial number of bits to
    count all the possible exponential initial states $I'$, the above procedure
    uses polynomial space, so $\csep_\exists(C_{\mathit{init}})$ and
    $\csep_\forall(C_{\mathit{init}})$ are PSPACE-complete.
\end{proof}

Let us now consider counterfactuals where the \emph{goal condition} can be
edited. Here, the existential and universal counterfactuals need to be addressed
separately.

\begin{theorem}
    \label{thm:goals:existential}
    $\csep_\exists(C_{\mathit{goal}})$ is
    \textbf{PSPACE-complete}.
\end{theorem}
\begin{proof}
    In this class of counterfactuals, changing the goal can be done by adding or
    removing truth assignments of the goal condition. Recall that, as observed
    in \cref{sec:classes}, for the existential counterfactuals, only
    \emph{weakening} the goal makes sense, \ie \emph{adding} truth assignments.
    
    Let $\theta\coloneqq \psi\land[P]$ (see \cref{prop:plantoltl}), and we test
    the satisfiability of $\theta$ (which can be done in polynomial space, as we
    recall). Then:
    \begin{enumerate}
        \item if $K=0$, and $\theta$ is satisfiable, $P$ is an existential 
            counterfactual for itself of cost zero, so we reply positively; 
            otherwise, we reply negatively;
        \item if $K>0$, and $\theta$ is satisfiable, then we extract the last
            state $s$ reached by the corresponding model of $\theta$, and we
            define $P'=\tuple{F,A,I,G'}$ where $G'\coloneqq G \lor s$; then we
            know a plan for $P'$ satisfying $\psi$ exists; if $\theta$ is
            unsatisfiable, we reply negatively.
    \end{enumerate}
    All of the above can be done in polynomial space, therefore
    $\csep_\exists(C_{\mathit{goal}})$ is PSPACE-complete. 
\end{proof}

The case of $\csep_\forall(C_{\mathit{goal}})$ is a bit more involved. Our
solution is shown in \cref{algo:goal:strengthening}. Intuitively, the algorithm
counts how many goal states are reachable by plans that do \emph{not} satisfy
$\psi$ (lines~\ref{algo:line:for} to \ref{algo:line:endfor}). This amount of
goal states have to be removed from the goal in any case, so we do not actually
\emph{store} them, we only count how many they are. Then, if the amount is
greater than the maximum allowed cost, we reply negatively (line
\ref{algo:line:greater}). Otherwise, we have to be careful. The definition of
universal counterfactual (\cref{def:problem}) forbids trivial solutions where
the problem has no plans at all (and hence trivially all their plans would
satisfy $\psi$). So we need to check whether there actually exists at least a
plan satisfying $\psi$. However, if we already reached the maximum allowed cost
(line \ref{algo:line:equal}), such a plan must reach the original goal
condition, because we do not have margin to add anything within the cost, so
return positively if and only if a plan of the original problem exists that
satisfies $\psi$ (line \ref{algo:line:equal:return}). If instead the amount of
assignments to be removed is strictly below the allowed cost (line
\ref{algo:line:below}), we have margin to ensure the set of solution plans is
not empty, so we reply positively if and only if there is a plan reaching any
goal state while satisfying $\psi$~(line \ref{algo:line:below:return}). Let us
now prove formally the soundness of \cref{algo:goal:strengthening}.

\begin{algorithm}[t]
    \caption{Algorithm for $\csep_\forall(C_{\mathit{goal}})$}
    \label{algo:goal:strengthening}
    \begin{algorithmic}[1]
        \Require $P=\tuple{F,A,I,G}$, $\psi$ \ltlf, $K\ge0$
        \State $\mathrm{counter}\gets0$
        \For{every truth assignment $\nu\models G$}\label{algo:line:for}
            \If{$[P]_\nu \land \neg\psi$ is satisfiable}
                \State $\mathrm{counter}\gets\mathrm{counter}+1$
            \EndIf
        \EndFor\label{algo:line:endfor}
        \If{$\mathrm{counter} > K$}\label{algo:line:greater}
            \State \Return \emph{false}\label{algo:line:greater:return}
        \ElsIf{$\mathrm{counter} = K$}\label{algo:line:equal}
            \State \Return whether $[P]_G \land \psi$ is satisfiable
            \label{algo:line:equal:return}
        \Else\label{algo:line:below}
            \State \Return whether $[P] \land \psi$ is satisfiable
            \label{algo:line:below:return}
        \EndIf
    \end{algorithmic}
\end{algorithm}

\begin{theorem}
    \label{thm:goals:universal}
    $\csep_\forall(C_{\mathit{goal}})$ is \textbf{PSPACE-complete}.
\end{theorem}
\begin{proof}
    We want to prove that \cref{algo:goal:strengthening} returns \emph{true} if
    and only if there exists a universal counterfactual scenario
    $\explanation{\forall}{P, \psi, C_{\mathit{goal}}}$ of cost less than or
    equal to $K$. So suppose that the latter is the case and let
    $\explanation{\forall}{P, \psi, C_{\mathit{goal}}}=\tuple{F,A,I,G'}$. In
    this proof, with some abuse of notation, we will denote as $G$ and $G'$ the
    \emph{set of goal states} that satisfy the $G$ and $G'$ formulas. As one can
    easily go from one representation to the other, this should not cause any
    ambiguity.

    Let us now define as $S\subseteq G$ the subset of goal states reached by any
    plan of $P$ that does \emph{not} satisfy $\psi$. Note that such goal states
    must in any case be removed from $G$ so the cost of
    $\explanation{\forall}{P, \psi, C_{\mathit{goal}}}$ is at least $\abs{S}$.
    Since the cost of $\explanation{\forall}{P, \psi, C_{\mathit{goal}}}$
    stays within $K$ we know $\abs{S}\le K$, otherwise we would need to overflow
    the cost to edit $G$ to $G'$. Therefore the $\mathrm{counter}$ variable at
    the end of the \emph{for loop} in \cref{algo:goal:strengthening} must be
    less than or equal to $K$ so the algorithm does \emph{not} reply
    \emph{false} in \cref{algo:line:greater:return}. So now we have two cases:
    \begin{enumerate}
        \item if $\abs{S}=K$, we know $G'=G\setminus S$, because we can only
        afford to remove $S$ and not to edit $G$ in any other way. Now, since
        $\explanation{\forall}{P, \psi, C_{\mathit{goal}}}$ is a universal
        counterfactual and by definition it has at least one plan satisfying
        $\psi$, the check at \cref{algo:line:equal:return} returns \emph{true}.
        \item If $\abs{S}<K$, then $G'=G\setminus S \cup S'$ where $S'$ is some
        set of goal states that $G'$ adds with regards to $G$. Then, since
        $\explanation{\forall}{P, \psi, C_{\mathit{goal}}}$ is a universal
        counterfactual and by definition it has at least one plan satisfying
        $\psi$, the check at \cref{algo:line:below:return} returns \emph{true}.
    \end{enumerate}

    \Viceversa, suppose \cref{algo:goal:strengthening} returns \emph{true}. We
    build a set $G'$ of goal states that will define the counterfactual
    $\explanation{\forall}{P, \psi, C_{\mathit{goal}}}=\tuple{F,A,I,G'}$. Let
    $S\subseteq G$ be the subset of goal states reached by any plan of $P$ that
    does \emph{not} satisfy $\psi$. Its size $\abs{S}$ is the value the variable
    $\mathrm{counter}$ holds at line~\ref{algo:line:endfor}. Since we return
    \emph{true} we know $\abs{S}\le K$, so we have two cases:
    \begin{enumerate}
        \item we return \emph{true} at line~\ref{algo:line:equal:return}, hence
        $\abs{S}=K$ and $[P]_G\land \psi$ is satisfiable. So we set
        $G'=G\setminus S$ to build our counterfactual: any plan reaching a goal
        state in $G'$ must satisfy $\psi$ (because it \emph{does not} satisfy
        $\neg\psi$), and we know the set of plans is not empty because
        $[P]_G\land \psi$ is satisfiable.
        \item We return \emph{true} at line~\ref{algo:line:below:return}, hence
        $\abs{S}<K$ and $[P]\land\psi$ is satisfiable. So there exists a plan of
        $P$ (not necessarily reaching $G$, see \cref{prop:plantoltl}) that
        satisfies $\psi$. Let $s$ be the goal state reached by such a plan. We
        set $G'=G\setminus S\cup\{s\}$ to build our counterfactual: in this way
        we ensure the set of plans satisfying $\psi$ is not empty, and we stay
        below the cost $K$.\qedhere
    \end{enumerate}
\end{proof}

Let us now focus on the last class of counterfactuals considered in
\cref{sec:classes}, namely those where \emph{actions' preconditions} are allowed
to change. As we will see, even though the existential counterfactuals are again
relatively easy to find (\ie still in polynomial space), \emph{universal}
counterfactuals may be more involved. We start with the easier case, which
nevertheless requires some additional background. Given a finite set of
propositions $F$, let $\omega:F\to\mathbb{N}$ be a \emph{weighting function} for
$F$. Then, given an \ltlf formula over $F$, one may ask for a word of
\emph{minimum weight}, where the weight of a word is the sum of the weights of
all the propositions true in any time step. Then, the following is known:
\begin{proposition}[\citet{DodaroFG22}]
    Deciding whether an \ltlf formula has a model of weight at most $k$ is
    \emph{PSPACE-complete}. A minimal model can be produced in polynomial space
    as well.
\end{proposition}

The above proposition is at the core of the following solution for
$\csep_\exists(C_{\mathit{act}})$.

\begin{theorem}
    \label{thm:acts:existential}
    $\csep_\exists(C_{\mathit{act}})$ is \textbf{PSPACE-complete}.
\end{theorem}
\begin{proof}
    Let $P=\tuple{F,A,I,G}$ be a planning problem, $\psi$ an \ltlf formula, and
    $K\ge0$. Recall from an observation in \cref{sec:classes} that it only makes
    sense to look for counterfactuals that \emph{weaken} the precondition of
    some actions. We proceed as follows. We define $A_{\mathit{relax}}$ as a set
    of actions equal to $A$ except that each action has its precondition set to
    \emph{true}. Then, we let $P_{\mathit{relax}}=\tuple{F,A\cup
    A_{\mathit{relax}}, I, G}$, and we define a weighting function $\omega$ such
    that $\omega(a)=0$ for all $a\in A$ and $\omega(a)=1$ for all $a\in
    A_{\mathit{relax}}$. Then, we consider the formula
    $\theta\coloneqq[P_{\mathit{relax}}]_G\land\psi$, and look for a model of
    \emph{minimal weight} according to $\omega$. Recall from
    \cref{sec:preliminaries} that this can be done in polynomial space. Then, we
    return \emph{true} if and only if $\theta$ is satisfiable and the minimal
    model has weight at most $K$.
    
    To see that the above procedure is correct, suppose on one hand that
    $\theta$ is indeed satisfiable with a minimal model of weight $L\le K$. We
    can extract each occurrence of an action $a$ from $A_{\mathit{relax}}$ in
    the model, and define as $S_a$ the set of all the states the action $a$ was
    applied to in the model. We then define our counterfactual
    $\explanation{\exists}{P, \psi, C_{\mathit{act}}}=\tuple{F,A',I,G}$ where
    $A'$ is defined as $A$ except that the precondition of each $a\in A$ is
    \emph{disjuncted} with all the states in $S_a$. Now, note that this is an
    existential counterfactual because, by construction, there is a plan of
    $\explanation{\exists}{P, \psi, C_{\mathit{act}}}$ satisfying $\psi$. It
    also has cost $L\le K$, by the way we defined $\omega$. 

    \Viceversa, suppose a counterfactual $\explanation{\exists}{P, \psi,
    C_{\mathit{act}}}=\tuple{F,A',I,G}$ exists with cost $L\le K$. Then, let
    $S_a$ be the set of states added to $\pre{a}$ in $A'$. Let $\pi$ be the plan
    that witnesses the counterfactual, \ie a plan that satisfies $\psi$. We can
    assume \emph{w.l.o.g.} that $\pi$ visits all the states in $S_a$ at least
    once, otherwise there would be a counterfactual of lower cost built by
    \emph{not adding} the non-visited state to $\pre{a}$. Then, we can construct
    a model of $\theta$ of minimum weight by setting proposition
    $a_{\mathit{relax}}$ or $a$ to true respectively each time an action $a$ is
    applied in $\pi$ to a state $s\in S_a$ or $s\not\in S_a$. Since each usage
    of an $a_{\mathit{relax}}$ has cost $1$, the cost of the model is $L\le
    K$.
\end{proof}

The complexity of $\csep_\forall(C_{\mathit{act}})$ concludes this section.
However, this case is more involved. To proceed, let us recall a standard
ingredient of complexity theory. \emph{Oracle machines} are a standard tool~(see
\eg~\citet{AroraB09}) for the study of the relationship between complexity
classes. An Oracle machine is a Turing machine with an associated \emph{oracle
language} $O$. At any time, the machine can query whether $s\in O$ for some
string $s$ and get the answer in \emph{constant time}. When we study the
computational complexity of a problem, by counting the oracle computation as a
single computation step we isolate the main algorithm solving the problem from
its sub-problems, clarifying the relationship between them. 

What we can show here is that $\csep_\forall(C_{\mathit{act}})$ belongs to the
class NEXP\textsuperscript{NP}, which is the class of problems solvable in
\emph{nondeterministic exponential time} when given access to an oracle for an
NP problem. This class is also often denoted as $\Sigma^{EXP}_2$, referring to
its second place in the \emph{exponential hierarchy}. Note that
$\text{NEXP}\subseteq\text{NEXP\textsuperscript{NP}}\subseteq\text{EXPSPACE}$
but it is not known whether $\text{NEXP}=\text{NEXP\textsuperscript{NP}}$ unless
$\text{P}=\text{NP}$.

In our case, the NP problem to be used as oracle is SAT, \ie \emph{propositional
satisfiability}, the prototypical NP-complete problem. We use a call to a SAT
oracle to test \emph{bounded \ltlf satisfiability}, \ie satisfiability of an
\ltlf formula restricted to models of bounded length. Many approaches exist in
literature to solve LTL and \ltlf satisfiability (or equivalently model
checking) through an iteration of SAT queries~\cite{ClarkeBRZ01,GeattiGMV24}. In
particular, we can state the following.
\begin{proposition}[\citet{GeattiGMV24}]
    \label{prop:ltlsat}
    Given an \ltlf formula $\psi$ and a bound $B>0$, in time polynomial in $B$
    and in the size of $\psi$ one can produce a \emph{propositional} formula
    $\enc(\psi, B)$ such that $\enc(\psi, B)$ is \emph{satisfiable} if and only
    if $\psi$ has a model of length at most $B$.
\end{proposition}

With \cref{prop:ltlsat} in place we can proceed.

\begin{theorem}
    \label{thm:acts:universal:membership}
    $\csep_\forall(C_{\mathit{act}})$ belongs to
    \textbf{NEXP\textsuperscript{NP}}.
\end{theorem}
\begin{proof}
    The solution is shown in \cref{algo:act:universal}. Following the definition
    of universal counterfactual, we non-deterministically guess a number of
    actions to be edited and the states to be added or removed and we build a
    candidate counterfactual $P'$. Then we need to check first if the problem
    has any solution plan at all, and then if all its solution plans satisfy
    $\psi$. For both checks we invoke the NP oracle, whose calls are underlined
    in \cref{algo:act:universal}:
    \begin{enumerate}
        \item the first check tests the satisfiability of
        $\enc([P']_G,2^{\abs{F}})$, which by \cref{prop:ltlsat} corresponds to
        testing the satisfiability of $[P']_G$ over models of length at most
        $2^{\abs{F}}$. We know this bound is sufficient because solution plans
        of classical planning problems do not have any reason to visit the same
        state twice, and after $2^{\abs{F}}$ steps we are sure to find a
        repetition.
        \item the second check tests the satisfiability of the formula
        $\enc([P']_G\land\neg\psi, 2^{\abs{F}})$ which by \cref{prop:ltlsat}
        corresponds to the satisfiability of $[P']_G\land\neg\psi$ over models
        of length at most $2^{\abs{F}}$, a bound that is chosen for the same
        reason outlined above. Since we accept when $[P']_G\land\neg\psi$ is
        \emph{not} satisfiable this corresponds to checking the validity of
        $[P']_G\to\psi$, which corresponds to the fact that all solution plans
        of $P'$ satisfy $\psi$.\qedhere
    \end{enumerate}
\end{proof}

\Cref{thm:acts:universal:membership} shows that
$\csep_\forall(C_{\mathit{act}})$ has a $\text{NEXP\textsuperscript{NP}}$ upper
bound, but a matching lower bound is still missing and not trivial to prove.
\Cref{thm:hardness} shows that the class of $\csep_\exists(C)$ and
$\csep_\forall(C)$ problems have a PSPACE \emph{core} that cannot be avoided,
made of the interconnection of two PSPACE-complete problems, namely classical
planning and \ltlf satisfiability. However, the possible-change relation $C$ can
be arbitrarily complex. For the practically relevant classes of counterfactuals
that we chose to study, the complexity has turned out to be relatively low, but
theoretically one may devise classes of possible-change relations that easily
cause complexity to increase arbitrarily, as evidence of the formal flexibility
of our framework.

\begin{algorithm}[t]
    \caption{Algorithm for $\csep_\forall(C_{\mathit{act}})$}
    \label{algo:act:universal}
    \begin{algorithmic}[1]
        \Require $P=\tuple{F,A,I,G}$, $\psi$ \ltlf, $K\ge0$
        \For{$i\gets 1$ to $K$}
            \State \textbf{guess} action $a$ \textbf{and} state $s$
            \State \textbf{guess} $\mathit{add}\in\{\mathit{true},\mathit{false}\}$
            \If{\emph{add}}
                \State $\pre{a}'\gets \pre{a}\lor s$
            \Else
                \State $\pre{a}'\gets \pre{a}\land\neg s$
            \EndIf
        \EndFor
        \If{\begin{tabular}[t]{@{\hspace{1em}}l}
                \underline{$\enc([P']_G,2^{\abs{F}})$ is satisfiable} and\\
                \underline{%
                    $\enc([P']_G\land\neg\psi, 2^{\abs{F}})$ is \emph{not}
                    satisfiable%
                }
            \end{tabular}\\\hspace*{-0.5ex}}
            \State \textbf{accept}
        \Else
            \State \textbf{reject}
        \EndIf
    \end{algorithmic}
\end{algorithm}

\section{\change{Discussion}}
\label{sec:discussion}

This paper introduced counterfactual scenarios, a novel explanation paradigm for
automated planning problems. Differently from existing proposals, our
explanations identify minimal modifications to a planning problem such that it
admits plans that comply with user-specified requirements, here captured by
means of \ltlf formulae. Given that our proposal is the first of its kind,
in this paper we set to investigate the computational complexity of generating
this type of explanations and explored potential theoretical barriers to our
proposal. 
Our analysis revealed that 
%
computing counterfactual scenarios is, in most cases, only as expensive as solving planning problems. This somewhat surprising result suggests that, for a significant class of problems, the additional step of identifying changes to a planning problem necessary to achieve a desired outcome does not introduce a prohibitive increase in complexity. 

\change{Our counterfactual scenarios are strongly related to the literature on model repair. For instance, work by Lin and Bercher addresses a setting where a flawed domain is repaired to fit a plan that serves as a witness~\cite{LinB21,LinB23}, often accounting for a notion of minimality of change~\cite{LinGSB25}. Although this line of work bears some similarities to our proposal, here we considered the more general problem of generating modifications that comply with an arbitrary \ltlf formula, rather than fixing existing domains based on specific plan examples. This is a major departure point, as our counterfactual scenarios can accommodate complex behaviours captured by abstract and highly expressive \ltlf formulae, and are guaranteed to return modified domains where all (or at least one) plans adhere to such behaviours. This profound technical difference is also reflected in the divergent computational complexities we obtained: most of our problems are PSPACE-complete, while the setting considered by Lin and Bercher typically yields NP-complete problems. Another critical difference between our work and model repair is the inherent plausibility requirement. Our counterfactuals are designed to generate actionable and plausible insights, pointing to what would need to be changed in a planning domain for a given behaviour to be observed. As such, changes suggested by our counterfactuals are not necessarily aimed at fixing a flawed domain; rather, they must be plausible (according to a domain-dependent notion of plausibility). This crucial requirement motivated our initial choice of using \ltlf formulae to capture plausibility and marks another fundamental difference from work on model repair.}

\change{Our work fits within the broader scope of explainable automated planning and introduces a novel explanation paradigm that takes a fundamentally different approach than both model reconciliation and contrastive explanations. Unlike model reconciliation, our counterfactual scenarios do not operate on the premise of reconciling a human's mental model with an agent's. Our approach focuses on identifying systematic modifications to a planning problem that would enable or enforce specific properties on the resulting plans, as captured by an \ltlf formula. As such, our counterfactuals are prescriptive and prospective, answering the question ``How can I change the planning problem so that plans satisfying a given property become possible?". This also marks a crucial difference from contrastive explanations, which have typically been defined to address the question: ``Why did the agent choose action A instead of action B in a plan?". Contrastive explanations are inherently local and retrospective, providing a justification for a specific choice made for a given planning problem. Our explanations have a global scope, identifying modifications that affect the entire planning problem space to enable or enforce a class of solutions characterized by an \ltlf formula. This is particularly evident in the universal setting, where there is a sharp contrast with the local and specific comparison inherent in contrastive explanations.
 }




\section{\change{Future work}}
\label{sec:future}

\change{The findings presented here have profound implications for the practical application of our counterfactuals, suggesting that existing planning algorithms could be leveraged to support the generation of counterfactual scenarios.} 
Our future work will focus on the development of \change{efficient algorithms} for the generation of counterfactual scenarios.
\change{Our} efforts will
initially concentrate on the classical planning setting and the use of \ltlf
specifications as described in this paper. \change{Given the strong link with model repair, we believe it would be interesting to investigate whether some of the algorithmic techniques developed in that space could be adapted to compute counterfactual scenarios, particularly if plausibility constraints could be integrated.} Once a robust algorithmic
foundation is established \change{for this setting}, it would be interesting to explore the generation of
counterfactual scenarios for different planning formalisms, such as those
addressing \change{numeric} domains, as well as considering alternative
specification languages over finite traces, such as CTL for probabilistic
systems~\cite{0020348}.


\section*{Acknowledgments}
Andrea Micheli has been supported by the STEP-RL project funded by the European Research Council under GA n. 101115870. Francesco Leofante was supported by Imperial College under the Imperial College Research Fellowship scheme. Nicola Gigante was partially funded by the NextGenerationEU FAIR PE0000013 project MAIPM (CUP C63C22000770006).

\bibliographystyle{kr}
\bibliography{refs}





\end{document}